\relax
\documentclass[letterpaper]{article}
\usepackage{times}
\usepackage{helvet}
\usepackage{courier}
\usepackage{named}

\usepackage{graphicx}
\usepackage{latexsym}

\usepackage{amsmath,amsthm}
\usepackage[ruled, vlined, longend, linesnumbered]{algorithm2e}
\usepackage{amssymb}

\usepackage{xcolor}

\DeclareMathOperator*{\argmin}{arg\,inf}
\newcommand{\ie}{\textit{i.e., }}
\newcommand{\eg}{\textit{e.g., }}

\newcommand{\dist}{p}
\newcommand{\cum}{F}
\newcommand{\dec}{G}
\newcommand{\quant}{q}
\newcommand{\lquant}{\underline\quant}
\newcommand{\uquant}{\overline\quant}
\usepackage{mathtools}
\DeclarePairedDelimiter\ceil{\lceil}{\rceil}
\DeclarePairedDelimiter\floor{\lfloor}{\rfloor}

\newtheorem{pro}{Proposition}
\newtheorem{ex}{Example}

\newtheorem{lem}{Lemma}
\newtheorem{defi}{Definition}

\begin{document}
%
\title{Optimizing Quantiles in \\
Preference-based Markov Decision Processes}


\author{
Hugo Gilbert\footnote{hugo.gilbert@lip6.fr}
\and
Paul Weng\footnote{paweng@cmu.edu} \and Yan Xu\footnote{xuyan@cmu.edu}
}


\maketitle
\begin{abstract}
In the Markov decision process model, policies are usually evaluated by expected cumulative rewards.
As this decision criterion is not always suitable, we propose in this paper an algorithm for computing a policy optimal for the quantile criterion.
Both finite and infinite horizons are considered.
Finally we experimentally evaluate our approach on random MDPs and on a data center control problem.
\end{abstract}

\section{Introduction}

Sequential decision-making in uncertain environments is an important task in artificial intelligence. 
Such problems can be modeled as Markov Decision Processes (MDPs). 
In an MDP, an agent chooses at every time step actions to perform according to the current state of the world in order to optimize a criterion in the long run. 
In standard MDPs, uncertainty is described by probabilities over the possible action outcomes, preferences are represented by numeric rewards and the expectation of future cumulated rewards is used as the decision criterion. 
And yet, for numerous applications, the expectation of cumulated rewards may not be the most appropriate criterion. 
For instance, in one-shot decision-making problems an alternative and well motivated objective for the agent is to insure a certain level of satisfaction with high probability. 

In this paper we focus on the decision criterion that consists in maximizing a quantile.
Intuitively, the $\tau$th quantile of a population is the value $x$ such that $100\cdot\tau$ percent of the population is equal or lower than $x$ and $100\cdot(1-\tau)$ percent of the population is equal or greater than $x$. 
Optimizing a quantile criterion offers nice properties: 
i) no assumption is made about the commensurability between preferences and uncertainty, 
ii) preferences over actions or trajectories can be expressed on a purely ordinal scale, 
iii) preferences induced over policies are more robust than with the standard criterion of maximizing the expectation of cumulated rewards.

As a result, maximizing a quantile is used in many applications. For instance, the \emph{Value-at-Risk} criterion \cite{Jorion06} widely used in finance is in fact a quantile. 
Moreover, in the Web industry \cite{WolskiBrevik14,DeCandiaHastorunJampaniKakulapatiLakshmanPilchinSivasubramanianVosshallVogels07}, decisions about performance or Quality-Of-Service are often made based on quantiles. 
For instance, Amazon reports \cite{DeCandiaHastorunJampaniKakulapatiLakshmanPilchinSivasubramanianVosshallVogels07} that they optimize the 99.9\% quantile for their cloud services.
More generally, in the service industry, because of skewed distributions \cite{BenoitVandenPoel09}, one generally does not want that customers are satisfied on average, but rather that most customers (\eg 99\% of them) to be as satisfied as possible.

\textbf{Our contribution:} 
We show that optimizing the quantile criterion amounts to solving a sequence of MDP problems using an Expected Utility criterion with a target utility function.
We provide a binary search algorithm using functional backward induction \cite{LiuKoenig06} as a subroutine for computing an optimal policy.  
Moreover, we investigate some properties of the optimal policies in the finite and infinite cases. Finally, we provide the results of experiments testing our algorithm in a variety of settings. 

The paper is organized as follows.
Section~\ref{sec:background} introduces the necessary background to present our approach and state formally our problem.
Section~\ref{sec:algo} presents the details of our solving algorithm for the finite horizon case.
Section~\ref{sec:infinite} provides some theoretical results in the infinite horizon case.
In Section~\ref{sec:expe}, we experimentally evaluate our proposition.
Section~\ref{sec:related} discusses the related work and
Section~\ref{sec:conclusion} concludes.

\section{Background}\label{sec:background}
In this section, we provide the background information necessary for the sequel.

\subsection{Markov Decision Process}

\textit{Markov Decision Processes} (MDPs) offer a general and powerful formalism to model and solve sequential decision-making problems \cite{Puterman94}.
An MDP is formally defined as a tuple $\mathcal{M}_T=(\mathcal{S},\mathcal{A},\mathcal{P},r,s_0)$ where $T$ is a time horizon, $\mathcal{S}$ is a finite set of states, $\mathcal{A}$ is a finite set of actions, $\mathcal{P}:\mathcal{S}\times\mathcal{A}\times\mathcal{S} \rightarrow \mathbb{R}$ is a transition function with $\mathcal{P}(s,a,s')$ being the probability of reaching state $s'$ when action $a$ is performed in state $s$, $r:\mathcal{S}\times \mathcal{A} \rightarrow \mathbb{R}$ is a bounded reward function and $s_0\in\mathcal{S}$ is a particular state called initial state. 

In a nutshell, at each time step $t$, the agent knows her current state $s_t$. 
According to this state, she decides to perform an action $a_t$. 
This action results in a new state $s_{t+1} \in \mathcal{S}$ according to probability distribution $\mathcal{P}(s_t,a_t,.)$, and a reward signal $r(s_t, a_t)$ which penalizes or reinforces the choice of this action. 
At time step $t=0$, the agent is in the initial state $s_0$. 
We will call $t$-history $h_t$ a succession of $t$ state-action pairs starting from state $s_0$ (\eg $h_t = (s_0,a_0,s_1,\ldots,s_{t-1},a_{t-1},s_t)$). 
We call episode a $T$-history and denote $\mathcal{E}$ the set of episodes. 

The goal of the agent is to determine a policy, \ie a procedure to select an action in a state, that is optimal for a given criterion. 
More formally, a {\em policy} $\pi$ at an horizon $T$ is a sequence of $T$ decision rules $(\delta_1,\ldots,\delta_{T})$. 
{\em Decision rules} are functions which prescribe the actions that the agent should perform. 
They are \textit{Markovian} if they only depend on the current state.  
Moreover, a decision rule is either {\em deterministic} if it always selects the same action in a given state or {\em randomized} if it can prescribe a probability distribution over possible actions. 
A policy can be \textit{Markovian}, \textit{deterministic} or  \textit{randomized} according to the type of its decision rules. 
Lastly, a policy is \textit{stationary} if it uses the same decision rule at every time step, \ie $\pi = (\delta,\delta,\ldots)$.

Different criteria can be defined in order to compare policies.
One standard criterion is {\em expected cumulated reward}, for which it is known that an optimal deterministic Markovian policy exists at any horizon $T$.
This criterion is defined as follows.
First, the value of a history $h_t = (s_0,a_0,s_1,\ldots,s_{t-1},a_{t-1},s_t)$ is described as the sum of rewards obtained along it, \ie $r(h_t) =  \sum_{i=0}^{t-1} r(s_{i}, a_i)$.
Then, the value of a policy $\pi = (\delta_1, \ldots, \delta_T)$ in a state $s$ is set to be the expected value of the histories that can be generated by $\pi$ from $s$.
This value, given by the {\em value function} $v^\pi_1 : \mathcal S \to \mathbb R$ can be computed iteratively as follows:
\begin{align}
      v^{\pi}_{T+1}(s)&= 0 \nonumber\\ 
v^{\pi}_t(s)&=  r(s,\delta_t(s)) + \sum_{s'\in \mathcal{S}}\mathcal{P}(s,\delta_t(s),s')v^{\pi}_{t+1}(s') \label{eq:value_function}
\end{align}

The value $v^{\pi}_t(s)$ is the expectation of cumulated rewards obtained by the agent if she performs action $\delta_t(s)$ in state $s$ at time step $t$ and continues to follow policy $\pi$ thereafter. 
The higher the values of $v^{\pi}_t(s)$ are, the better. 
Therefore, value functions induce a preference relation $\succsim_{\pi}$ over policies in the following way:
\begin{align*}
\pi \succsim_{\pi} \pi' \Leftrightarrow \forall s\in \mathcal{S}, \forall t=1,\ldots, T, v^{\pi}_t(s) \geq v^{\pi'}_t(s)
\end{align*}

A solution to an MDP is a policy, called {\em optimal policy}, that ranks the highest with respect to $\succsim_{\pi}$. Such a policy can be found by solving the \textit{Bellman equations}.
\begin{align*}
      v^{*}_{T+1}(s)&= 0 \\ 
      v^{*}_t(s)&=\max_{a\in \mathcal{A}}  r(s,a)+ \sum_{s'\in \mathcal{S}}\mathcal{P}(s,a,s')v^{*}_{t+1}(s')
\label{eq:bellman}
\end{align*}
As can be seen, the preference relation $\succsim_{\pi}$ over policies is directly induced by the reward function $r$.

The decision criterion, based on the expectation of cumulated rewards, may not always be suitable.
Firstly, unfortunately, in many cases, the reward function $r$ is not known.
One can therefore try to uncover the reward function by interacting with an {\em expert} of the domain considered \cite{ReganBoutilier09,WengZanuttini13}. 
However, even for an expert user, the elicitation of the reward function can be burdensome. 
Indeed, this process can be cognitively very complex as it requires to balance several criteria in a complex manner and as it can imply a large number of parameters. 
In this paper, we address this problem by only assuming that we have a strict weak ordering on episodes. 

Secondly, for numerous applications, the expectation of cumulated reward, as used in Equation~\ref{eq:value_function}, may not be the most appropriate criterion (even when a numeric reward function is defined). 
For instance, in the Web industry, most decisions about performance are based on the minimal quality of $99\%$ of the possible outcomes. 
Therefore, in this article we aim at using a quantile (defined in Section~\ref{sec:quantile}) as a decision criterion to solve an MDP. 

\subsection{Preferences over Histories}

For generality's sake, contrary to standard MDPs, we define in this work the reward function to take values in a set $\mathcal R$.
Moreover, we assume that the values of histories take values in a set $\mathcal W$, called the wealth level space, and that the value of a history $h_t = (s_0, a_0, s_1, \ldots, s_t)$ is defined by:
\begin{align*}
w(h_0) = w_0 \hspace{0.5cm} w(h_t) = w(h_{t-1}) \circ r(s_{t-1}, a_{t-1})
\end{align*}
where $h_{t-1} = (s_0, a_0, s_1, \ldots, s_{t-1})$, $\circ$ is a binary operation from $\mathcal W \times \mathcal R$ to $\mathcal W$ and $w_0 \in \mathcal W$ is the left identity element of $\circ$.
Let $\mathcal W_T \subset \mathcal W$ be the set of wealth levels of $T$-histories.
We make three assumptions about $\mathcal W_T$:
\begin{itemize}
\item It is ordered by a total order $\preceq_{\mathcal W}$, which defines how $T$-histories are compared,
\item It admits a lowest element, denoted $w_{\min}$ and a greatest element, denoted $w_{\max}$ for order $\preceq_{\mathcal W}$.
\item A distance consistent with $\preceq_{\mathcal W}$ is defined over $\mathcal W_T$. It is denoted $d(w, w')$ for any pair $(w, w') \in \mathcal W_T \times \mathcal W_T$.
\end{itemize}
Note that when a distance is defined, for any pair $(w, w')$, its set of mid-elements is also defined $mid(w, w') = \argmin\{ \max(d(w, w''), d(w', w'')) \,|\, w'' \in \mathcal W_T \}$.

In a numerical context, the possible wealth levels of a state are the possible sums (resp. $\gamma$-discounted sums) of rewards that can be obtained during an episode. 
We have $w_{\max} = R_{\max} T$ (resp. $w_{\max} = R_{\max} \frac{(1-\gamma)^T}{1-\gamma}$) with $R_{\max}$ being the highest reward and $mid(w, w') = \{(w+w')/2\}$.
In the most general case, the possible wealth levels of a state are the possible histories (or more precisely their equivalent classes) that can be obtained during an episode. 
Here, if the equivalence classes are known and denoted by $w_1 \prec_{\mathcal W} w_2 \prec_{\mathcal W} \ldots \prec_{\mathcal W} w_m$ and if $d(w_i,w_j)=|j-i|$, then $w_{\min} = w_1$, $w_{\max} = w_m$ and $mid(w_i, w_j) = \{w_{\floor{(i+j)/2}}, w_{\ceil{(i+j)/2}}\}$ (where $\floor{x}$is the greatest integer smaller than $x$ and $\ceil{x}$ is the smallest integer greater than $x$).

The goal of the agent is then to make sure that most of the time, it will generate episodes that have the highest possible wealth levels. 
This can be implemented by optimizing a quantile criterion as explained in the next subsection.

\subsection{Quantile Criterion}\label{sec:quantile}

Intuitively, the $\tau$-quantile of a population of ordered elements, for $\tau\in[0,1]$, is the value $\quant$ such that $100 \cdot \tau\%$ of the population is equal or lower than $\quant$ and $100 \cdot (1-\tau)\%$ of the population is equal or greater than $\quant$. 
The $0.5$-quantile, also known as the median, can be seen as the ordinal counterpart of the mean.
More generally, quantiles define decision criteria that have the nice property of not requiring numeric valuations, but only an order.
They have been axiomatically studied as decision criteria by \citeauthor{Rostek10} \shortcite{Rostek10}.

We now give a formal definition of quantiles.
For this purpose we define the probability distribution $\dist^{\pi}$ over wealth levels induced by a policy $\pi$, 
\ie $\dist^{\pi}(w)$ is the probability of getting a wealth level $w \in \mathcal W_T$ when applying policy $\pi$ from the initial state.
The {\em cumulative distribution} induced by $\dist^\pi$ is then defined as $\cum^{\pi}$ where $\cum^{\pi}(w) = \sum_{w' \preceq_{\mathcal W} w} \dist^{\pi}(w')$ is the probability of getting a wealth level not preferred to $w$ when applying policy $\pi$.  
Similarly, the {\em decumulative distribution} induced by $\dist^\pi$ is defined as $\dec^{\pi}(w) = \sum_{w \preceq_{\mathcal W} w'} \dist^{\pi}(w')$ is the probability of getting a wealth level ``not lower'' than $w$. 

These two notions of cumulative and decumulative enable us to define two kinds of criteria. 
First, given a policy $\pi$, we define the lower $\tau$-quantile for $\tau \in(0,1]$ as:
\begin{equation}
\lquant_{\tau}^{\pi} = \min\{w \in \mathcal W_T \,|\, \cum^{\pi}(w) \geq \tau\}
\end{equation}
where the $\min$ operator is with respect to $\prec_{\mathcal W}$.

Then, given a policy $\pi$, we define the upper $\tau$-quantile for $\tau \in[0,1)$ as:
\begin{equation}
\uquant_{\tau}^{\pi} = \max\{w \in \mathcal W_T \,|\, \dec^{\pi}(w) \geq 1-\tau\}
\end{equation}
where the $\max$ operator is with respect to $\prec_{\mathcal W}$.

If $\tau = 0$ or $\tau = 1$ only one of $\lquant_{\tau}^{\pi}$ or $\uquant_{\tau}^{\pi}$ is defined and we define the $\tau$-quantile $\quant_{\tau}^{\pi}$ as that value. 
When both are defined, by construction, we have $\lquant_{\tau}^{\pi} \preceq_{\mathcal W} \uquant_{\tau}^{\pi}$.
If those two values are equal, $\quant_{\tau}^{\pi}$ is defined as equal to them. 
For instance, this is always the case in continuous settings for continuous distributions. 
However, in our discrete setting, it could happen that those values differ, as shown by Example~\ref{ex:1}. 
\begin{ex}\label{ex:1}
Consider an MDP where $\mathcal W_T = \{w_1 \prec_{\mathcal W} w_2 \prec_{\mathcal W} w_3\}$.  
Now assume a policy $\pi$ attains each wealth level with probabilities $0.5$, $0.2$ and $0.3$ respectively. 
Then it is easy to see that  $\lquant_{0.5}^{\pi}  = w_1$ whereas $\uquant_{0.5}^{\pi} = w_2$.
\end{ex}
When the lower and upper quantiles differ, one may define the quantile as a function of the lower and upper quantiles \cite{Weng12}.
For simplicity, we show in this paper how to optimize (approximately) the lower and the upper quantiles.

\begin{defi}
A policy $\pi^*$ is optimal for the lower (resp. upper) $\tau$-quantile criterion if:
\begin{align}
\lquant_\tau^{\pi^*} = \max_\pi \lquant_\tau^{\pi} \hspace{0.5cm} (\text{resp. }\uquant_\tau^{\pi^*} = \max_\pi \uquant_\tau^{\pi})
\end{align}
where the $\max$ operator is with respect to $\prec_{\mathcal W}$ and taken over all policies $\pi$ at horizon $T$.
\end{defi}

Even in a numerical context where a numerical reward function is given and the quality of an episode is defined as the cumulative of rewards received along the episode, this criterion is difficult to optimize, notably due to the two following related points:
\begin{itemize}
\item It is {\em non-linear} meaning for instance that the $\tau$-quantile $\quant_{\tau}^{\tilde{\pi}}$ of the mixed policy $\tilde{\pi}$ that generates an episode using policy $\pi$ with probability $p$ and $\pi'$ with probability $1-p$ is not given by $p \quant_{\tau}^{\pi} + (1-p)\quant_{\tau}^{\pi'}$.
\item It is {\em non-dynamically consistent} meaning that at time step $t$, an optimal policy computed in $s_0$ with horizon $T$ might not prescribe in state $s_t$ to follow a policy optimal in $s_t$ for horizon $T-t$. 
\end{itemize}
Three solutions are then possible \cite{McClennen90}: 
1) adopting a {\em consequentialist} approach, \ie at each time step $t$ we follow an optimal policy for the problem with horizon $T-t$ and initial state $s_t$ even if the resulting policy is not optimal at horizon $T$; 
2) adopting a {\em resolute choice} approach, \ie at time step $t=0$ we apply an optimal policy for the problem with horizon $T$ and initial state $s_0$ and do not deviate from it;
3) adopting a {\em sophisticated resolute choice} approach \cite{Jaffray98,FargierJeantetSpanjaard11}, \ie we apply a policy $\pi$ (chosen at the beginning) that trades off between how much $\pi$ is optimal for all horizons $T, T-1, \ldots, 1$.

With non-dynamically consistent preferences, it is debatable to adopt a consequentialist approach, as the sequence of decisions may lead to dominated results. 
In this paper, we adopt a resolute choice point of view.
We leave the third approach for future work.

As optimizing exactly a (lower or upper) quantile is hard, we aim at finding an approximate solution.
Let $\lquant^*_\tau$ and $\uquant^*_\tau$ be equal to the optimal lower and upper quantile respectively.

\begin{defi}
Let $\varepsilon>0$.
A policy $\pi^*_\varepsilon$ is said to be $\varepsilon$-optimal for the lower (resp. upper) $\tau$-quantile criterion if $d(\lquant_{\tau}^{\pi^*_\varepsilon} ,\lquant_{\tau}^{*}) \leq \varepsilon$ (resp. $d(\uquant_{\tau}^{\pi^*_\varepsilon} ,\uquant_{\tau}^{*}) \leq \varepsilon$).
\end{defi}

\section{Solving Algorithm}\label{sec:algo}

In this section, we present a technique for computing an $\varepsilon$-optimal policy for the quantile criterion.
Our approach amounts to solving a sequence of MDPs optimizing EU with target utility functions (see Section~\ref{sec:dp}).

\subsection{Binary Search}
In order to justify our algorithm, we introduce two lemmas that characterize the optimal lower and upper quantiles\footnote{For lack of space, all proofs are in the supplementary material.}:
\begin{lem}
The optimal lower $\tau$-quantile $\lquant^*$ satisfies:
\begin{align}
\lquant^* &= \min\{w : \cum^{*}(w) \geq \tau \} \label{eq:bslq1}\\
\cum^{*}(w) &= \min_\pi \cum^{\pi}(w) \quad \forall w \in \mathcal W \label{eq:pblq1}
\end{align}
\label{lem:3}
\end{lem}
\noindent Note the last two equations can be equivalently rewritten: 
\begin{align}
\lquant^* &= \min\{w : \dec_\prec^*(w) \leq 1 - \tau \} \label{eq:bslq}\\
\dec_\prec^{*}(w) &= \max_\pi \dec_\prec^{\pi}(w) \quad \forall w \in \mathcal W \label{eq:pblq}
\end{align}
where $\dec_\prec^{\pi}(w) = 1 - \cum^\pi(w) = \sum_{w \prec_{\mathcal W} w'} \dist^\pi(w')$.
\begin{lem}
The optimal upper $\tau$-quantile $\uquant^*$ satisfies:
\begin{align}
\uquant^* &= \max\{w : \dec^{*}(w) \geq 1 - \tau \} \label{eq:bsuq}\\
\dec^{*}(w) &= \max_\pi \dec^{\pi}(w)  \quad \forall w \in \mathcal W \label{eq:pbuq}
\end{align}
\label{lem:4}
\end{lem}
Given Lemmas~\ref{lem:3} and \ref{lem:4} the problem now reduces to finding the right value of $w \in \mathcal{W}$ that solves the problems defined by Equation~\ref{eq:bslq} or \ref{eq:bsuq}. 
Our solving method is based on binary search (see Algorithm~\ref{alg:lqo}) and on the function $solve(\mathcal M, w)$ that returns a pair $(\pi, p)$, the solution of the problems defined by Equation~\ref{eq:pblq} or \ref{eq:pbuq} for a fixed $w$, \ie the $\max$ is equal to $p$ and attained at $\pi$. 
Note that while for the upper quantile criterion, $solve(\mathcal M, \uquant_\tau^{\pi^*})$ returns an optimal policy, 
for the lower quantile, $solve(\mathcal M, \lquant_\tau^{\pi^*})$ may not  if $\lquant_\tau^{\pi^*} \!\succ_\mathcal{W}\! \min(\mathcal{W}_T)$. 
However, $solve(\mathcal M, prec(\lquant_\tau^{\pi^*}))$ returns an optimal policy where $prec(w) $ is the most preferred element such that $prec(w)\!\prec_\mathcal{W}\! w$ (see supplementary material).

In the next subsection, we show how  function $solve$ can be computed for the lower and upper quantile.

Note that when $\mathcal{W}_T$ is defined on the real line, Algorithm \ref{alg:lqo} needs only  
\[\ceil*{\log_2 d(w_{\max}, w_{\min})/\varepsilon}\] iterations to terminate by using  $[w_{\min},w_{\max}]$ as $\mathcal{W}_T$.
In the case where $\mathcal W_T$ is finite, binary search can of course determine the optimal policy with $\varepsilon = 1$ and needs $\ceil{\log_2(|\mathcal W_T|)}$ iterations.

\begin{algorithm}[t]
\small
\DontPrintSemicolon
\KwData{MDP $\mathcal{M}$, $\tau$, $\varepsilon$}
\KwResult{an $\varepsilon$-optimal policy $\pi$}
$\overline w\gets w_{\max}$; 
$\underline w\gets w_{\min}$; 
$w\leftarrow mid(\underline w, \overline w)$\\

\While{$d(\overline{w}, \underline{w}) > \varepsilon$}{
$(\pi, p) =  solve(\mathcal{M}, w)$;\\
\eIf{$p > 1 - \tau$ (resp. $p \geq 1 - \tau$)}{ 
$\underline w\gets w$;
$w\leftarrow \max(mid(\underline w, \overline w))$;
$\pi^* \leftarrow \pi$;
}
{
$\overline w \gets w$;
$w\gets \min(mid(\underline w, \overline w))$;
}
}
\Return $\pi^*$
\caption{Binary Search for the Lower Quantile (resp. Upper Quantile)}
\label{alg:lqo}
\end{algorithm}

The next proposition asserts that Algorithm~\ref{alg:lqo} is correct:

\begin{pro}
Algorithm \ref{alg:lqo} returns an $\varepsilon$-optimal policy for the lower (or upper) quantile criterion.
\end{pro}

\subsection{Dynamic Programming}\label{sec:dp}

For $\triangleleft \in \{\prec_\mathcal W, \preceq_\mathcal W\}$, we denote by $U_{w}^\triangleleft: \mathcal W \to \mathbb R$ the function, called {\em target utility function}, defined as follows:
\begin{equation}\label{eq:u}
U_{w}^\triangleleft(x) = 1 \text{ if } w \triangleleft x \text{ and } 0 \text{ else.}
\end{equation}

When optimizing the lower (resp. upper) quantile, function $solve(\mathcal M, w)$ can be computed by solving MDP $\mathcal M$ using EU as a decision criterion with $U_w^{\prec_\mathcal W}$ (resp. $U_w^{\preceq_\mathcal W}$) as a utility function. 
Indeed, we have:
\begin{align*}
\mathbb E_\pi[ U^\triangleleft_w \big(w(H_T)\big) ] &= \mathbb P[ w \triangleleft w(H_T) \,|\,\pi ] 
\end{align*}
where $H_T$ is a random variable representing a $T$-history and $\mathbb P[ w \triangleleft w(H_T) \,|\,\pi]$ denotes the probability that $\pi$ generates a history whose wealth is strictly better (resp. at least better) than $w$ when $\triangleleft = \prec_\mathcal W$ (resp. $\triangleleft = \preceq_\mathcal W$).

Following \cite{LiuKoenig06}, this problem can be solved with a functional backward induction (Algorithm \ref{alg:biuf}). 
For each state $s$, it maintains a function $V_{t}(s,.)$ which associates to each possible wealth level $w$ the expected utility obtained by applying an optimal policy in state $s$ for the remaining $T-t$ time steps with $w$ as initial wealth level. 
At each time step $(t=T,\ldots,1)$ this function is updated similarly as in backward induction except that operations are not applied to scalars but to functions.
The $\max$ and $\times$ operations are extended over functions as pointwise operations.
As utility functions defined by Equation~\ref{eq:u} are piecewise-linear, $V_t(s,.)$ is also piecewise-linear because all the operations in Line~\ref{algline:update} of Algorithm~\ref{alg:biuf} preserve this property.

\begin{algorithm}[t]
\small
\DontPrintSemicolon
\KwData{MDP $\mathcal{M}$, wealth $w$}
\KwResult{an optimal policy $\pi$}
\For{all $s\in S$}{$V_{T+1}(s,.)\leftarrow U^\triangleleft_w(.)$}
\For{$t=T$ to $1$}{
	\For{all $s\in S$}{
$V_t(s, \cdot )\leftarrow \displaystyle\max_a \sum_{s'\in S}\mathcal P(s,a,s')V_{t+1}(s', \cdot \circ r(s,a))$ \label{algline:update}
				}
			}  
\Return~$(\pi_{V_1}, V_1(s_0, w_0))$ $\backslash\backslash$ $\pi_{V_1}$= policy corresponding to $V_1$
\caption{FunctionalBackwardInduction}
\label{alg:biuf}
\end{algorithm}

Policies returned by Algorithm~\ref{alg:biuf} have a special structure.
They are deterministic and wealth-Markovian:
\begin{defi}
A policy is said to be {\em wealth-Markovian} if its decision rules are functions of both the current state and the current wealth level.
\end{defi}

Besides, this is also the case for policies optimal with respect to the quantile criterion.
\begin{pro}\label{pro:opt}
Optimal policies for the lower or upper quantile at horizon $T$ can be found as deterministic wealth-Markovian policies.
\end{pro}

\section{Infinite Horizon}\label{sec:infinite}
We present in this section some results regarding the infinite horizon case.
Similarly to the finite horizon setting, the situation for the quantile criterion is not as simple as for the standard case. 
Indeed, in the infinite horizon case, it may happen that there is no {\em stationary} deterministic Markovian policy that is optimal (w.r.t. the quantile criterion) among all policies, contrary to standard MDPs.
\begin{ex}
Consider an MDP with two states $s_1$ and $s_2$ and two actions $a_1$ and $a_2$. In $s_1$, the transition probabilities are $\mathcal P(s_1, a_1, s_1) = 0.1$, $\mathcal P(s_1, a_1, s_2) = 0.9$ and $\mathcal P(s_1, a_2, s_2) = 1$.
To make this example shorter, we assume that rewards depend on next states.
The rewards are $r(s_1, a_1, s_1) = 1$, $r(s_1, a_1, s_2) = -1$ and $r(s_1, a_2, s_2) = 1$.
In $s_2$, the transition probabilities are $\mathcal P(s_2, a_1, s_2) = \mathcal P(s_2, a_2, s_2) = 1$.
Rewards are null for both actions in $s_2$. 
Among all decision rules, there are only two distinct rules: $\delta_1(s_1) = a_1$ and $\delta_2(s_1) = a_2$.
To ensure that the values of histories are well-defined, we assume that they are defined as discounted sum of rewards with a discount factor $\gamma=0.9$.
One can then check that the $0.95$-quantile of the stationary policy using $\delta_1$ is $0.1$, that of the stationary policy using $\delta_2$ is $1$.
Finally, the $0.95$-quantile of the policy applying first $\delta_1$ and then $\delta_2$ is $1.9$.
Therefore, no stationary deterministic Markovian policy is optimal for the quantile criterion.
\end{ex}

\begin{figure}[tb!]
\begin{center}
\includegraphics[width=0.8\textwidth]{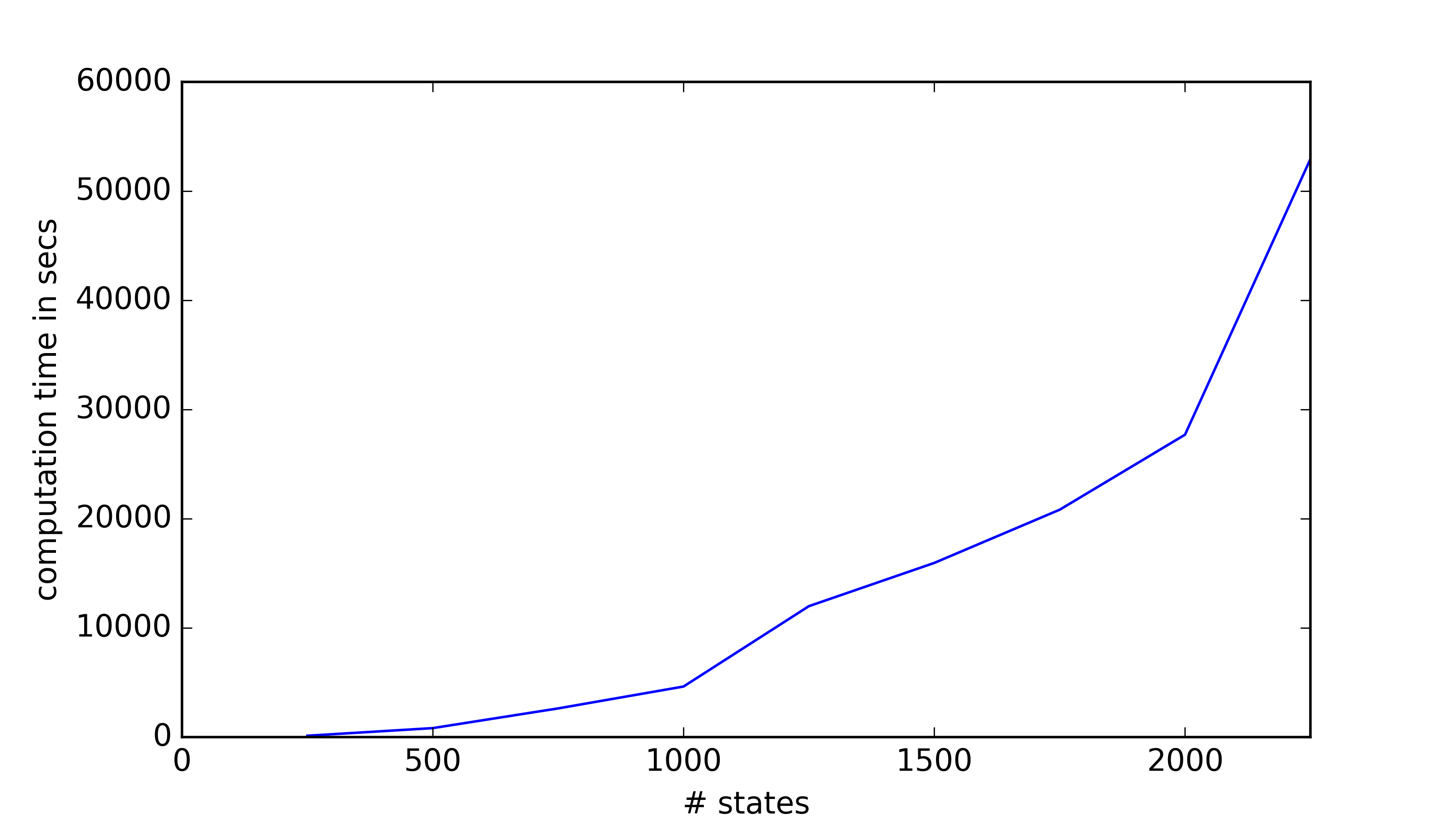}
\caption{Computation times vs state sizes for Functional Backward Induction.}
\label{1 BI: Running Time Cost}
\end{center}
\end{figure}

However, considering wealth-Markovian policies, some results can be given when rewards are numeric and wealth levels are undiscounted:

\begin{pro}
Optimal policies for the lower or upper quantile can be found as stationary deterministic wealth-Markovian policies in the two following cases:
\begin{description}
\item[(i)] 
$\forall (s, a) \in \mathcal{S}\times\mathcal{A}, r(s,a) \le 0$.
\item[(ii)] 
$\forall (s, a) \in \mathcal{S}\times\mathcal{A}, r(s,a) \ge 0$. Furthermore, we require the existence of a finite upper bound on the optimal lower and upper quantiles.
\end{description}

\end{pro}

Then, a solving algorithm can be obtained from Algorithm \ref{alg:lqo} by replacing functional backward induction (Alg.~\ref{alg:biuf}) by functional value iteration \cite{LiuKoenig06} in the binary search. This amounts to do the for loop over $t$ (line 4) until convergence of $V_{t}$, \ie $\| V_{t} - V_{t-1}\|_\infty \le \epsilon'$. Binary search will then return an $(\epsilon + \epsilon')$-optimal for the $\tau$-quantile. However, note that in the first (resp. second) case, a lower (resp. upper) bound on the optimal lower or upper quantile is required to do the binary search.

\section{Experimental Results}\label{sec:expe}
We experimentally evaluated our approach on a server equipped with four Intel(R) Xeon(R) CPU E5-2640 v3 @ 2.60GHz and 64Gb of RAM.
The algorithms were implemented in Matlab and ran only on one core.
We expect the running times to be improved with a more efficient programming language and by exploiting a multicore architecture.

\begin{figure}[tb!]
\begin{center}
\includegraphics[width=0.8\textwidth]{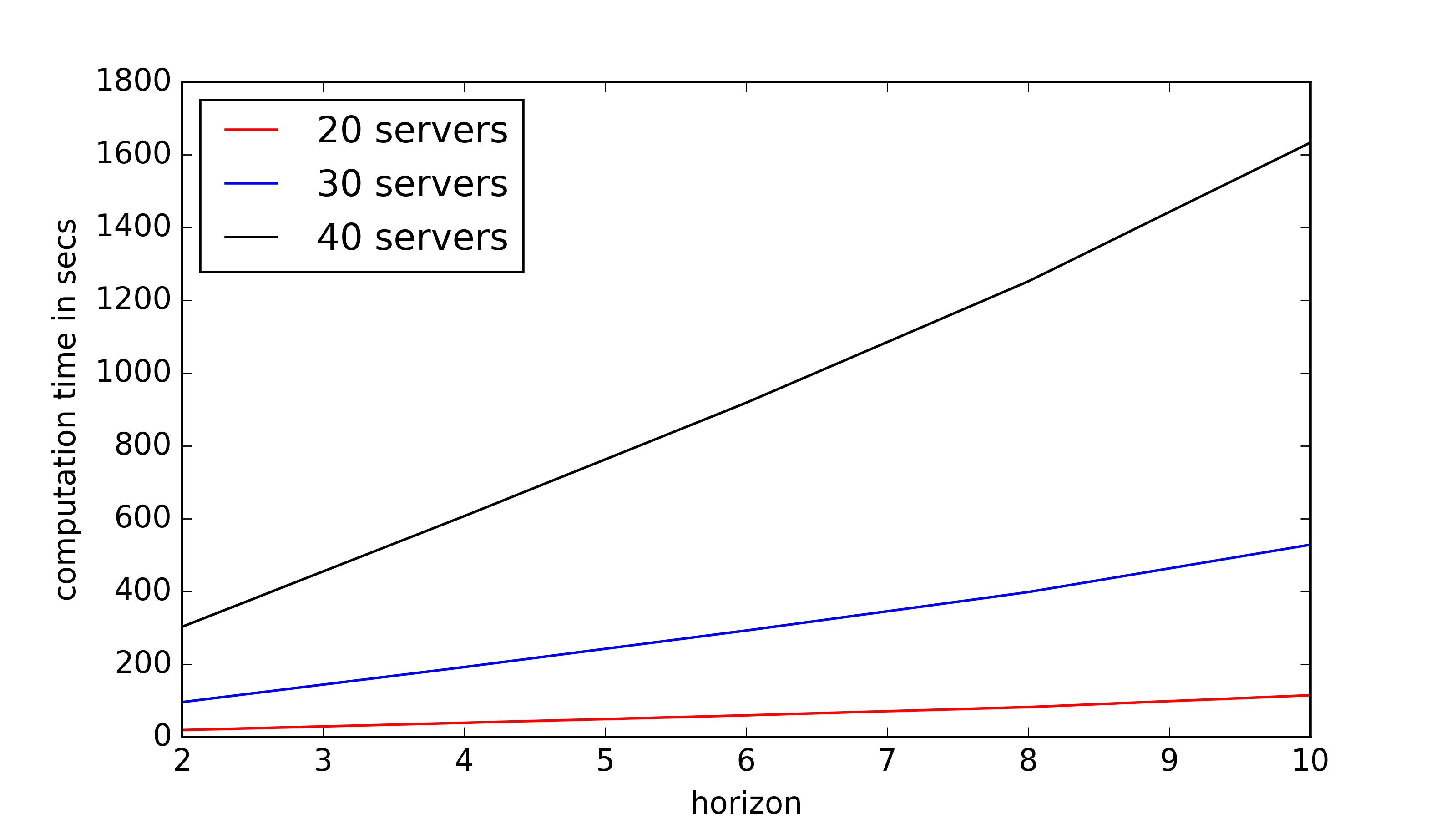}
\caption{Computation times vs horizon for Functional Backward Induction.}
\label{1.5 BI: Running Time Cost}
\end{center}
\end{figure}

We designed three sets of experiments.
Although our approach could be used in a preference-based setting, we performed the experiments with numerical rewards for simplicity.
The first shows the running time of functional backward induction for different varying state sizes on random MDPs. 
The second set of experiments shows the running time of functional backward induction for different horizons on a data center control problem with various number of servers. 
Finally, the third compares the cumulative distributions of a policy optimal for the quantile criterion and a policy optimal for the standard criterion on a fixed MDP.

The first set of experiments was conducted on Garnets \cite{McKinnonThomas95}, which  designate random MDPs with a constrained branching factor.
A Garnet $G(n_S, n_A, b)$ is characterized by $n_S$ a number of states, $n_A$ a number of actions and $b$ the number of successor states for every state and action.
For our experiments, $n_S\in\{ 250, 500, 750, 1000, 1250, 1500, 1750, 2000, 2250 \}$ and we set $n_A = 5$ and $b = \ceil{\log_2 n_S}$.
Rewards are randomly chosen in $[0, 1]$ and the values of histories are simply cumulated rewards.
The horizon of the problem was set to $5$. 
The results are presented in Figure \ref{1 BI: Running Time Cost} where the x-axis represents the state size and the y-axis the computation time.  
Each point is the average over 10 runs.
Naturally, computation times increases with state sizes.
In this setting, binary search would call functional backward induction $\ceil*{\log_2(1/\varepsilon)} = 10$ times if $\varepsilon=10^{-3}$.

The second set of experiments was performed on a more realistic domain, which is a data center control problem inspired by the model proposed by \citeauthor{YinSinopoli14} \shortcite{YinSinopoli14}.
In this problem, one needs to decide every time step how many servers to switch on or off, while maximizing Quality-of-Service and minimizing power consumption.
In the model proposed by \citeauthor{YinSinopoli14}, the two objectives are simply combined into one cost, which defines our reward signal.
The state is defined as the number of servers that are currently on and the number of jobs that needs to be processed during a time step.
The action represents the number of servers that will be on at the next time step.
We assume for simplicity that the maximum number of jobs that can arrive at one timestep is three times the total number of servers.
For instance, in a problem with $n=30$ servers, the total number of states is $30\times3\times30 = 2700$.
Besides, the distribution of the next number of jobs is modeled as a Poisson distribution whose parameter can be $\ceil{n/2}$, $\ceil{3n/2}$ or $\ceil{5n/2}$ (to model different regimes) depending on the current number of jobs.
Figure~\ref{1.5 BI: Running Time Cost} shows the computation times of functional backward induction for $n\in\{20, 30, 40\}$ and different horizons.
We can see that for more structured problems, the computation time is much more reasonable than on random MDPs.

In the last set of experiments, to give an intuition of the kind of policy obtained when optimizing a quantile, we compare the cumulative distribution of a policy optimal for the quantile criterion and that of a policy optimal for the standard criterion.
This experiment is performed on an instance of Garnet $G(100, 5,  \ceil{\log_2 100})$ whose rewards are slightly modified to make the distribution of the optimal policy skewed, as it is often the case in some real applications \cite{BenoitVandenPoel09}.
The horizon is set to $5$ and we optimize the $0.1$-quantile with $\varepsilon=0.001$ in binary search.
The two cumulative distributions are plotted in Figure \ref{2 Comparison: quantile and standard}.
We can observe that although the optimal policy for the standard criterion maximizes the expectation, it may be a risky policy to apply as the probability of obtaining a high reward is low.
On the contrary, the optimal policy for the $\tau$-quantile criterion will guarantee a reward as high as possible with probability at least $1-\tau$.

\begin{figure}[!tb]
\begin{center}
\includegraphics[width=0.8\textwidth]{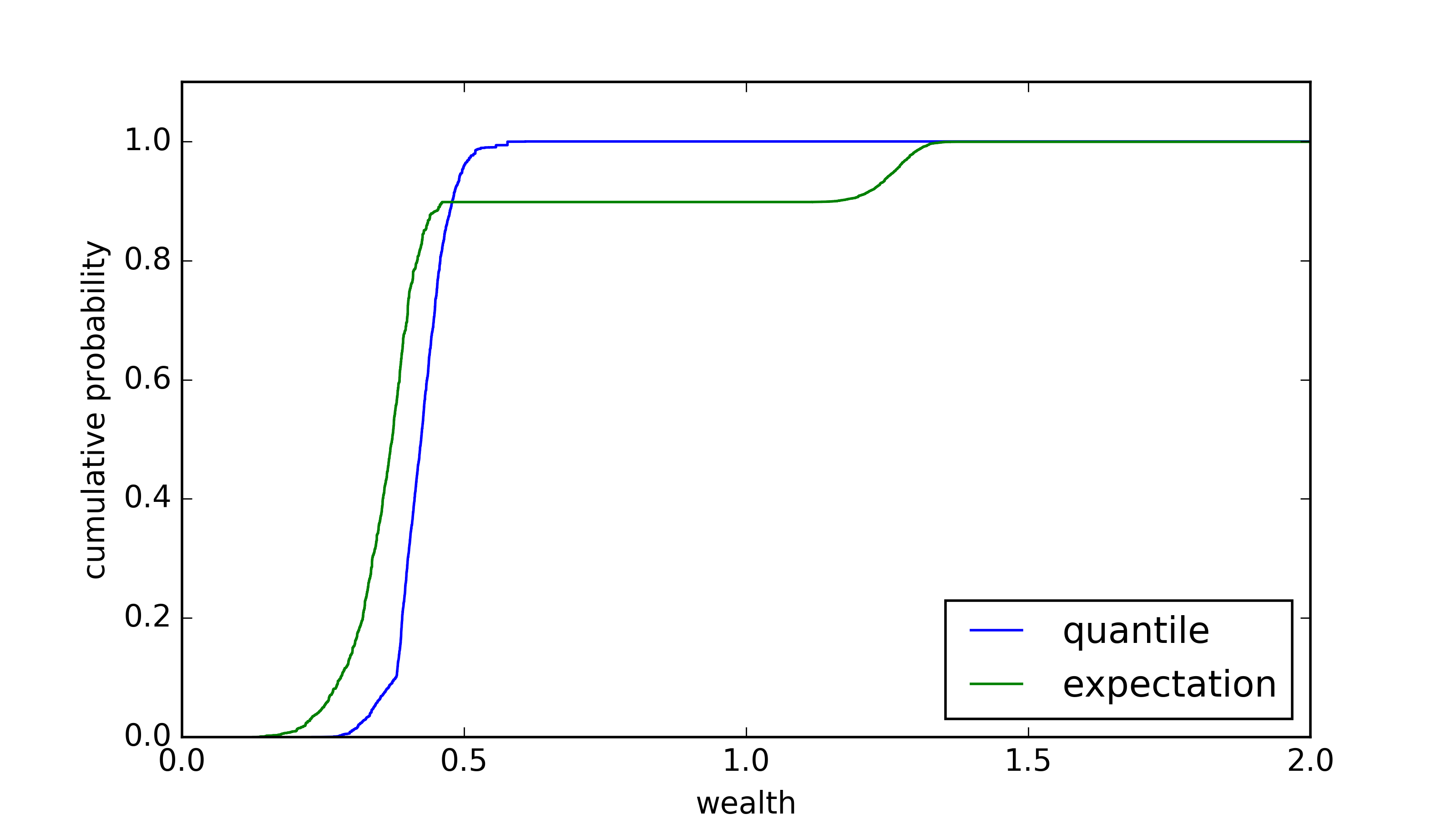}
\caption{Comparison of cumulative distributions under the quantile criterion and standard criterion}
\label{2 Comparison: quantile and standard}
\end{center}
\end{figure}

\section{Related Work}\label{sec:related}
Much work in the MDP literature \cite{BBMSWeng10} considered decision criteria different to the standard ones (i.e., expected discounted sum of rewards, expected total rewards or expected average rewards).
For instance, in the operations research community, 
\citeauthor{White87} \shortcite{White87} considered different cases where preferences over policies only depend on sums of rewards: Expected Utility (EU), probabilistic constraints and mean-variance formulations.
In this context, he showed the sufficiency of working in a state space augmented with the sum of rewards obtained so far.
Recently, \cite{prashanth2013actor} and \cite{mannor2011mean} provided algorithms for this mean-variance formulation. 
\citeauthor{FilarKallenbergLee89} \shortcite{FilarKallenbergLee89} investigated decision criteria that are variance-penalized versions of the standard ones. 
They formulated the obtained optimization problem as a non-linear program. 
Several researchers \cite{White93,BouakizKebir95,Yu98,WuLin99,OhtsuboToyonaga02,hou2014revisiting,fan2005arriving} worked on the problem of optimizing the probability that the total (discounted) reward exceeds a given threshold.

Additionally, in the artificial intelligence community, \cite{LiuKoenig05,LiuKoenig06,ermon2012probabilistic} also investigated the use of EU as a decision criterion in MDPs. In the continuation of this work, \citeauthor{GilbertSpanjaardViappianiWeng15} \shortcite{GilbertSpanjaardViappianiWeng15} investigated the use of Skew-Symmetric Bilinear (SSB) utility \cite{Fishburn81} functions --- a generalization of EU with stronger descriptive abilities --- as decision criteria in finite-horizon MDPs. 
Interestingly, SSB also encompasses probabilistic dominance, a decision criterion that can be employed in preference-based sequential decision-making \cite{BusaFeketeSzorenyiWengChengHullermeier14}.

Recent work in MDP and reinforcement learning considered  conditional Value-at-risk (CVaR), a criterion related to quantile, as a risk measure.
\citeauthor{BauerleOtt11} \shortcite{BauerleOtt11} proved the existence of deterministic wealth-Markovian policies optimal with respect to CVaR. 
\citeauthor{ChowGhavamzadeh14} \shortcite{ChowGhavamzadeh14} proposed gradient-based algorithms for CVaR optimization.
In contrast, \citeauthor{BorkarJain14} \shortcite{BorkarJain14} used CVaR in inequality constraints instead of as objective function.

Closer to our work, several quantile-based decision models have been investigated in different contexts. 
In uncertain MDPs where the parameters of the transition and reward functions are imprecisely known, \citeauthor{DelageMannor07} \shortcite{DelageMannor07} presented and investigated a quantile-like criterion to capture the trade-off between optimistic and pessimistic viewpoints  on an uncertain MDP. 
The quantile criterion they use is different to ours as it takes into account the uncertainty present in the parameters of the MDP.
\citeauthor{FilarKrassRoss95} \shortcite{FilarKrassRoss95} proposed an algorithm for optimizing the quantile criterion when histories are valued by average rewards.
In that setting, they showed that an optimal stationary deterministic Markovian policy exists.
In MDPs with ordinal rewards \cite{Weng11,Weng12,Filar83}, quantile-based decision models were proposed to compute policies that maximize a quantile using linear programming.
While quantiles in those works are defined on distributions over ordinal rewards, we defined them as distributions over histories. 

More recently, in the machine learning community, quan\-tile-based criteria have been proposed in the multi-armed bandit (MAB) setting, a special case of reinforcement learning.
\citeauthor{YuNikolova13} \shortcite{YuNikolova13} proposed an algorithm in the pure exploration setting for different risk measures, including Value-at-Risk. 
\citeauthor{CarpentierValko14} \shortcite{CarpentierValko14} studied the problem of identifying arms with extreme payoffs, a particular case of quantiles.
Finally, \citeauthor{SzorenyiBusaFeketeWengHullermeier15} \shortcite{SzorenyiBusaFeketeWengHullermeier15} investigated MAB problems where a quantile is optimized instead of the mean.

\section{Conclusion}\label{sec:conclusion}
In this paper we have developed a framework to solve sequential decision problems in a very general setting according to a quantile criterion. Modeling those problems as MDPs we developed an offline algorithm in order to compute an $\epsilon$-optimal policy and investigated the properties of the optimal policies in the finite and infinite horizon cases. 
Lastly, we provided experimental results, testing those two algorithms in a variety of settings.

As future work, we plan to investigate how this work can be extended to the case of reinforcement learning, a framework more involved than the one of MDPs where the dynamics of the problems are unknown and must be learned.

\fontsize{9.5pt}{10.pt} \selectfont
\bibliography{biblio160226}
\bibliographystyle{named}
\normalsize

\newpage
\onecolumn
\section{Supplementary material of ``Optimizing Quantiles in Preference-based Markov Decision Processes" }

We provide in this section the proofs of our lemmas and propositions.

\setcounter{lem}{0}
\begin{lem}
The optimal lower $\tau$-quantile $\lquant^*$ satisfies:
\begin{align*}
\lquant^* &= \min\{w : \cum^{*}(w) \geq \tau \} \\
\cum^{*}(w) &= \min_\pi \cum^{\pi}(w) \quad \forall w \in \mathcal W 
\end{align*}
\end{lem}
\begin{proof}
We recall that for any policy $\pi$, $F^{\pi}$ is nondecreasing and that consequently $F^*$ is also nondecreasing. 
Let $w_1 = \max_\pi\min_w\{w\in\mathcal{W}_T|F^{\pi}(w)\geq\tau\}$ and let $w_2 =  \min\{w : \cum^{*}(w) \geq \tau \}$.
By contradiction, assume $w_1 > w_2$. Then there exists $\pi$ such that $F^{\pi}(w_1)\geq\tau$ and $F^{\pi}(w)<\tau$, $\forall w < w_1$. Thus $F^{\pi}(w_2)<\tau$ which contradicts the definition of $w_2$. Now, assume $w_2 > w_1$. Then $F^*(w_1)<\tau$. Thus, there exists $\pi$ such that $F^{\pi}(w_1)<\tau$ and $\underline{q}_{\tau}^{\pi} > w_1$ which contradicts the definition of $w_1$.
\end{proof}

\begin{lem}
The optimal upper $\tau$-quantile $\uquant^*$ satisfies:
\begin{align*}
\uquant^* &= \max\{w : \dec^{*}(w) \geq 1 - \tau \} \\
\dec^{*}(w) &= \max_\pi \dec^{\pi}(w)  \quad \forall w \in \mathcal W 
\end{align*}
\end{lem}
\begin{proof}
We recall that for any policy $\pi$, $G^{\pi}$ is nonincreasing and that consequently $G^*$ is also nonincreasing. 
Let $w_1 = \max_\pi\max_w\{w\in\mathcal{W}_T|G^{\pi}(w)\geq1 - \tau\}$ and let $w_2= \max\{w : \dec^{*}(w) \geq 1 - \tau \}$.
By definition of $w_1$, there exists a policy $\pi$ such that $G^\pi(w_1)\geq 1-\tau$, thus $G^*(w_1)\geq 1-\tau$ and $w_2 \geq w_1$.
By definition of $w_2$, there exists a policy $\pi$ such that $G^\pi(w_2)\geq 1-\tau$, thus $\max_w\{w\in\mathcal{W}_T|G^{\pi}(w)\geq1 - \tau\}\geq w_2$ and $w_1 \geq w_2$.
\end{proof}

The following example shows that $\cum^*(\lquant^*)$ (see Equation~\ref{eq:pblq1}) may not be attained by an optimal policy (for the lower quantile):
\begin{ex}
Let $\cum_1$ and $\cum_2$ be two cumulatives defined over three elements $w_1 \prec_{\mathcal W} w_2 \prec_{\mathcal W} w_3$ with the following probabilities:
$\cum_1 = (0.5, 0.5, 1)$ and $\cum_2 = (0, 0.6, 1)$.
The lower $0.5$-quantile of $\cum_1$ is $w_1$ and that of $\cum_2$ is $w_2$. 
Therefore the optimal lower quantile is $\lquant^* = w_2$.
We have $\cum^* = (0, 0.5, 1)$ and $\cum^*(\lquant^*) = 0.5$, which is attained by $\cum_1$.
\end{ex}
This implies that $solve(\mathcal M, \lquant^*)$ may return a non-optimal policy when $\lquant^* \succ_{\mathcal W} \min(\mathcal W_T)$. 
For $w\in \mathcal{W}_T$, we define $prec(w)$ as the most preferred element of $\mathcal{W}_T$ such that $prec(w) \prec w$. 
If there are no element $w' \in \mathcal{W}_T$ such that $w' \prec w$,  $prec(w)$ is defined as $w$.
The optimal policy can be found using the following property:
\begin{lem}
Any policy $\pi^*$ such that $\cum^{\pi^*}(prec(\lquant^*)) = \cum^{*}(prec(\lquant^*))$ is an optimal policy with regard to the lower quantile criterion.
\end{lem}
\begin{proof}
Assume that $\lquant^* \succ_{\mathcal W} \min(\mathcal W_T)$. 
Otherwise the lemma is clearly true.
Assume by contradiction that there is a non-optimal policy $\pi$ such that $\cum^{\pi}(prec(\lquant^*)) = \cum^{*}(prec(\lquant^*))$.
Let $\lquant$ be the lower $\tau$- quantile of policy $\pi$, $\lquant^*$ be the optimal lower quantile and $\pi^*$ be an optimal policy.
By assumption, we have $\lquant \preceq_{\mathcal W} prec(\lquant^*) \prec_{\mathcal W} \lquant^*$ and $\cum^{\pi^*}(prec(\lquant^*)) \ge \cum^{\pi}(prec(\lquant^*))$.
As a cumulative is non-decreasing, we have $\cum^{\pi}(prec(\lquant^*)) \ge \cum^\pi(\lquant) \ge \tau$, which contradicts the fact that the lower quantile of $\pi^*$ is $\lquant^*$.
\end{proof}

Before proving that Algorithm~\ref{alg:lqo} is correct, we introduce a lemma that gives sufficient conditions for a policy to be approximately optimal. 
\begin{lem}
Let $\pi$ be a policy for which there exists $w$ such that $d(w,\lquant_{\tau}^{*}) \leq \varepsilon$ (resp. $d(w,\uquant_{\tau}^{*}) \leq \varepsilon$) and:
\begin{align*}
\cum^{\pi}(w) < \tau \hspace{0.5cm} (\text{resp. } \dec^{\pi}(w) \ge 1 - \tau ).
\end{align*}
Then $\pi$ is $\varepsilon$-optimal for the lower (resp. upper) $\tau$-quantile criterion.
\label{lem:epsopt}
\end{lem}
\begin{proof}
Indeed, for such a policy, as $\cum^{\pi}(w)$ is nondecreasing (resp. $\dec^{\pi}(w)$ is nonincreasing), we have that $\lquant_\tau^{\pi}\in[w,\lquant_\tau^{*}]$ (resp. $\uquant_\tau^{\pi}\in[w,\uquant_\tau^{*}]$) and thus $d(\lquant_\tau^{\pi}, \lquant_\tau^{*}) \leq \varepsilon$ (resp. $d(\uquant_\tau^{\pi}, \uquant_\tau^{*}) \leq \varepsilon)$ .
\end{proof}

\setcounter{pro}{0}
\begin{pro}
Algorithm \ref{alg:lqo} returns an $\varepsilon$-optimal policy for the lower (or upper) quantile criterion.
\end{pro}
\begin{proof}
If $\mathcal{W}_T\subset \mathbb{R}$ we have seen that the algorithm terminates in $\ceil*{\log_2\frac{d(w_{\max}, 0)}{\varepsilon}}$ iterations. 
In the most general setting, the algorithm terminates, because in the worst case we will check all the $m$ possible final wealth values.
Let $\pi$ be the policy returned by the algorithm.
For the lower (resp. upper) quantile, when the algorithm terminates, $d(\lquant_{\tau}^{*},\underline{w})$ (resp. $d(\uquant_{\tau}^{*},\underline{w})$) $\leq d(\overline{w},\underline{w}) \leq \varepsilon$ and $\cum^{\pi^*}(\underline{w}) < \tau$ (resp. $\dec^{\pi}(\underline{w}) \ge 1 - \tau$). Thus, we can apply Lemma \ref{lem:epsopt} which concludes the proof. 
\end{proof}

\begin{pro}
Optimal policies for the lower or upper quantile at horizon $T$ can be found as deterministic wealth-Markovian policies.
\end{pro}
\begin{proof}
We recall that for the lower (resp. upper) quantile criterion, procedure $solve(\mathcal{M},w)$ returns the policy which minimizes $F^\pi(w)$ (resp. maximizes $G^\pi(w)$).
Thus, for any policy $\pi$, by definition of quantiles, $solve(\mathcal{M},prec(\lquant_\tau^{\pi}))$ (resp. $solve(\mathcal{M},\uquant_\tau^{\pi})$) returns a deterministic wealth-Markovian policy, which is at least as good as $\pi$ regarding the lower (resp. upper) quantile criterion. 
As the set of deterministic wealth-Markovian policies is finite in the finite horizon case, taking the one with highest lower (resp. upper) quantile concludes the proof.
\end{proof}

\begin{pro}
Optimal policies for the lower or upper quantile in the infinite horizon setting can be found as stationary deterministic wealth-Markovian policies in the two following cases:
\begin{description}
\item[(i)] 
$\forall (s, a) \in \mathcal{S}\times\mathcal{A}, r(s,a) \le 0$.
\item[(ii)] 
$\forall (s, a) \in \mathcal{S}\times\mathcal{A}, r(s,a) \ge 0$. Furthermore, we require the existence of a finite upper bound on the optimal lower and upper quantiles.
\end{description}
\end{pro}
\begin{proof}
We prove for the upper quantile and case (i), the other cases are similar.
If all policies have $-\infty$ as quantile, they are all optimal.
Now, if one policy has a finite lower quantile $q\in\mathbb{R}^{-}$, the optimal quantile must be greater than or equal to $q$.
From the original MDP, consider the state-augmented MDP whose state space is defined by $\overline{\mathcal{S}} =\{(s,w)|s,w \in \mathcal{S}\times\mathcal{W}\}$.
In $\overline{\mathcal{S}}$, regroup all states having a wealth level strictly less than $q$ in a single absorbing state. 
Indeed, as the optimal upper quantile is greater than $q$ and $r(s,a) \le 0$, $\forall s,a$, the choices of the policies in those states are irrelevant to find an optimal policy w.r.t the upper quantile criterion. 
Note that the resulting augmented state space $\overline{\mathcal{S}}_{<q}$ is finite. 
In this MDP, we use reward functions parametrized by a value $x \in \mathcal{W}$ defined as follows : 
\begin{align*}
r_x((s,w),a) =   
   \left \{
   \begin{array}{r c l}
      -1 & & \text{ if } w \ge x \text{ and } w + r(s,a) < x \\
      0  & & \text{else.} \\
   \end{array}
   \right .
\end{align*}
A policy solving this MDP w.r.t the expectation of total reward criterion maximizes the probability of getting an episode with a wealth level greater than or equal to $x$. 
According to Puterman (\citeyear{Puterman94}, Theorem 7.1.9), such a policy can be found as a stationary deterministic Markovian policy in the augmented MDP. 
Stated differently, there exists a stationary deterministic wealth-Markovian optimal policy in the original MDP. 
Then, for any policy $\pi$, the stationary deterministic wealth-Markovian policy which is optimal when using reward function $r_{\uquant_\tau^{\pi}}$ (and expectation of total reward) is at least as good as $\pi$ regarding the upper quantile criterion. 
By partitioning those policies by regrouping the ones that agree on $\overline{\mathcal{S}}_{<q}$ we reduce the set of stationary deterministic wealth-Markovian policies to a finite set. 
By taking the ``best one'' in this set, we obtain a stationary deterministic wealth-Markovian optimal policy.
\end{proof}

\end{document}